\documentclass{article}
\usepackage{spconf}

\usepackage{epsfig}
\usepackage{subcaption}
\usepackage{calc}
\usepackage{amssymb}
\usepackage{amstext}
\usepackage{amsmath}
\usepackage{amsthm}
\usepackage{pslatex}
\usepackage{tikz}
\usepackage{pgfplots}
\usepackage{amsmath,amsfonts,bm}

\def\1{\bm{1}}

\def\vzero{{\bm{0}}}

\def\vbeta{{\bm{\beta}}}
\def\vlambda{{\bm{\lambda}}}
\def\vrho{{\bm{\rho}}}
\def\va{{\bm{a}}}

\def\vu{{\bm{u}}}

\def\vx{{\bm{x}}}

\def\mI{{\bm{I}}}

\def\mM{{\bm{M}}}

\def\mR{{\bm{R}}}
\def\mS{{\bm{S}}}

\def\mPhi{{\bm{\Phi}}}

\def\mPhi{{\bm{\Phi}}}

\DeclareMathAlphabet{\mathsfit}{\encodingdefault}{\sfdefault}{m}{sl}
\SetMathAlphabet{\mathsfit}{bold}{\encodingdefault}{\sfdefault}{bx}{n}
\newcommand{\tens}[1]{\bm{\mathsfit{#1}}}

\def\tX{{\tens{X}}}

\def\gA{{\mathcal{A}}}
\def\gB{{\mathcal{B}}}

\def\gN{{\mathcal{N}}}

\def\gS{{\mathcal{S}}}
\def\gT{{\mathcal{T}}}

\def\gW{{\mathcal{W}}}

\def\sC{{\mathbb{C}}}

\def\sR{{\mathbb{R}}}

\newcommand{\spec}{\mathrm{Sp}}

\DeclareMathOperator{\Tr}{Tr}

\newcommand{\xml}{\hat{\vx}_{\sf ML}}
\usepackage{empheq}
\DeclareMathOperator{\asto}{\xrightarrow{\text{a.s.}}}
\DeclareMathOperator{\supp}{sup}

\theoremstyle{plain}
\newtheorem{theorem}{Theorem}[section]

\newtheorem{corollary}[theorem]{Corollary}

\theoremstyle{definition}
\newtheorem{definition}[theorem]{Definition}
\newtheorem{assumption}[theorem]{Assumption}
\theoremstyle{remark}
\usepackage{comment}
\usepackage{flushend}

\usepackage{todonotes}
\setuptodonotes{inline}
\usepackage{color}

\usepackage[nospace]{varioref}

\def \as {\xrightarrow[]{\text{a.s.}}}

\makeatletter
\def\blfootnote{\xdef\@thefnmark{}\@footnotetext}
\makeatother

\begin{document}

\title{On the Accuracy of Hotelling-Type Asymmetric Tensor Deflation:\\ A Random Tensor Analysis}

\name{Mohamed El Amine Seddik$^1$, Maxime Guillaud$^2$, Alexis Decurninge$^3$ and José Henrique de M.~Goulart$^4$}
\address{$^1$Technology Innovation Institute, Abu Dhabi, UAE \quad
$^2$Inria / CITI Lab, Lyon, France\\
$^3$Huawei Technologies, Boulogne-Billancourt, France \quad
$^4$IRIT, Toulouse INP, CNRS, Toulouse, France}

\maketitle

\abstract{This work introduces an asymptotic study of Hotelling-type tensor deflation in the presence of noise, in the regime of large tensor dimensions.
Specifically, we consider a low-rank asymmetric tensor model of the form $\sum_{i=1}^r \beta_i\gA_i + \gW$ where $\beta_i\geq 0$ and the $\gA_i$'s are unit-norm rank-one tensors such that $\left| \langle \gA_i, \gA_j \rangle \right| \in [0, 1]$ for $i\neq j$ and $\gW$ is an additive noise term.
Assuming that the dominant components are successively estimated from the noisy observation and subsequently subtracted, we leverage recent advances in random tensor theory in the regime of asymptotically large tensor dimensions to analytically characterize the estimated singular values and the alignment of estimated and true singular vectors at each step of the deflation procedure. Furthermore, this result can be used to construct estimators of the signal-to-noise ratios $\beta_i$ and the alignments between the estimated and true rank-1 signal components.}


\blfootnote{J.~H.~de M.~Goulart was supported by the ANR LabEx CIMI (ANR-11-LABX-0040) within the French Programme ``Investissements d’Avenir.'' M.~Guillaud was supported by the French ``Future Networks'' Priority Research Programme and Equipment (PEPR) NF-PERSEUS.}

\section{\uppercase{Introduction}}\label{sec:introduction}

The problem of finding {the best rank-one} approximation of a tensor has been introduced in \cite{Lathauwer000bestrank1} and further studied in e.g. \cite{kofidis2002bestrank1_supersymmetric,kolda2011shifted} for the case of finite-size tensors.
More recently, the analysis of random tensors of asymptotically large dimensions has attracted significant attention, following the analysis in \cite{montanari2014statistical} of the spiked (rank-1 plus noise) model of the form $\beta \, \vx^{\otimes d} + \gW/\sqrt{n}$ where $\vx\in \sR^n$ is some high-dimensional unit vector referred to as a \textit{spike}, $\gW$ is a \textit{symmetric} random tensor of order $d$ having independent Gaussian entries and $\beta\geq 0$ is a parameter controlling the signal-to-noise ratio, in the large-dimensional regime $n\to \infty$.
%The analysis of random tensors has attracted significant attention in the last decade since the introduction of the concept of tensor principal component analysis (PCA), which generalizes principal component analysis to high-order arrays. A simple model introduced in \cite{montanari2014statistical} is the so-called spiked tensor model of the form $\beta \, \vx^{\otimes d} + \gW/\sqrt{n}$ where $\vx\in \sR^n$ is some high-dimensional unit vector referred to as a \textit{spike}, $\gW$ is a \textit{symmetric} random tensor of order $d$ having independent Gaussian entries and $\beta\geq 0$ is a parameter controlling the signal-to-noise ratio.
%\todo{H : Petit détail : $\gW$ est pluttôt donné par la symétrisation d'un tenseur à entrées gaussiennes standard (donc la variance de chaque entrée dépend du nombre d'indices répétés).}

Follow-up results \cite{perry2016statistical,lesieur2017statistical,jagannath2020statistical,chen2021phase,goulart2021random} have improved the understanding of the behavior of the spiked model and allowed to identify statistical and/or algorithmic
thresholds for nontrivial signal recovery in the asymptotic regime.
%\todo{H : Je trouve cette phrase précédente un peu floue. Plutôt statistical and algorithmic thresholds for signal recovery ? D'ailleurs, on est dans tous les cas dans le régime à grande dimensions non ? (je ne comprends pas ce "in particular" qui vient après).}
We briefly summarize their main findings as follows: for $d\geq 3$, it has been shown that there exists a statistical threshold $\beta_{stat}=O(1)$ to the signal-to-noise ratio below which it is information-theoretically impossible to recover or even detect the spike, while above $\beta_{stat}$ recovery is theoretically possible.
% by studying the critical points of the quadratic loss. 
% \todo{Je ne comprends pas le "by studying the critical points of the quadratic loss." Je pense que si c'est pour résumer tous les travaux cités ensemble, on peut en rester à "theoretically possible" (les méthodes de preuve varient).}
Moreover, the asymptotic alignment $\langle \vx, \xml \rangle$ between $\vx$ and the maximum likelihood solution $\xml$ was computed as a function of $\beta$ and shown to be information-theoretically optimal \cite{jagannath2020statistical}.
% \todo{Pas vrai pour n'importe quel point critique, non ? On peut se restreindre au max local "informatif".} 
Besides, since almost all tensor problems including computing $\xml$ are NP-hard \cite{hillar2013most}, many researchers were interested in identifying the algorithmic threshold for $\beta$ above which recovery could be possible with a polynomial-time algorithm. In particular, the work \cite{montanari2014statistical} introduced a method for estimating $\vx$ based on tensor unfolding and showed that spike recovery is possible above the algorithmic threshold $\beta_{\text{algo}}=O(n^{ \frac{d-2}{4} })$.

These ideas were further generalized to the \textit{asymmetric} spiked tensor model of the form $\beta \vx_1\otimes \cdots \otimes \vx_d + \gW/\sqrt{\sum_\ell n_\ell}$, where $\vx_\ell\in \sR^{n_\ell}$ are unit vectors and $\gW$ is a random tensor with standard Gaussian i.i.d.\@ entries. Specifically, \cite{arous2021long} provided an analysis of the unfolding method for asymmetric tensors and determined its algorithmic threshold to be $\beta_{\text{algo}}=O(n^{ \frac{d-2}{4} })$ when $n_i=n$ for all $i$, while the analysis of \cite{seddik2021random} suggests the existence of a statistical threshold $\beta_{\text{s}} = O(1)$ above which a local solution $\vu_\ell$ of the maximum-likelihood estimator (MLE) of the dominant rank-1 component non-trivially aligns 
% \todo{on peut préciser ce "aligns"?} 
% \todo{H : Porquoi pas MLE du spike ? Pour aligns, on pourrait dire "nontrivially aligns" ?} 
with the signal, and further quantified the asymptotic alignment $\langle \vx_\ell, \vu_\ell \rangle$.

\textbf{Contribution:} In this paper, we address the extension of these ideas to a more general setting, namely asymmetric \emph{low-rank} spiked tensors of the form $\sum_{i=1}^r \beta_i \vx_{i, 1}\otimes \cdots \otimes \vx_{i, d} + \gW/\sqrt{\sum_\ell n_\ell}$ where $\vx_{i,\ell}\in \sR^{n_\ell}$ are unit vectors.
Specifically, we consider the study of a simple deflation procedure, first introduced by Hotelling in the context of matrix principal component analysis \cite{Hotelling1933} and still used in modern applications such as the recent AlphaTensor model \cite{fawzi2022discovering}, which consists in iterated rank-one approximation followed by subtraction of the estimated rank-one component.

Differently from the matrix PCA, the rank-one terms $\beta_i \, \vx_{i, 1}\otimes \cdots \otimes \vx_{i, d}$ in our model are \emph{not} pairwise orthogonal, and therefore in general deflation cannot possibly recover them exactly, even in the absence of noise \cite{stegeman2010subtracting}.
Our main goal is to characterize analytically the accuracy of the approximate decomposition obtained through deflation, in terms of the alignments of the estimated terms with the true ones, and as a function of the cross-term alignments $\langle \vx_{i,\ell}, \vx_{j,\ell} \rangle \in (0, 1)$ for $i\neq j$, while the analysis in the orthogonal case (i.e.~when $\langle \vx_{i,\ell}, \vx_{j,\ell} \rangle=0$ for $i\neq j$) trivially boils down to the rank-one model studied in \cite{chen2021phase,seddik2021random}. Precisely, we characterize this behavior by computing the alignments $\langle \vx_{i,\ell}, \vu_{j,\ell} \rangle$ in the asymmetric case, in the high-dimensional regime when $n_\ell\to \infty$, using the random matrix approach developed in \cite{goulart2021random, seddik2021random}.

\section{Notation and Background}
Scalars (resp. vectors) are denoted by lowercase (resp.~bold lowercase) letters. Tensors are denoted as in $\gA$, $\gT$. We set $[n] \equiv \{1, \ldots, n\}$, and $\asto$ denotes almost sure convergence.
% The unit sphere in $\sR^p$ is denoted by $\sS^{p-1}$.

\subsection{\uppercase{Tensor Notation and Contractions}}\label{sec_tensor_notations}
In this section, we provide the main tensor notation and definitions used throughout the paper. 

\textbf{Order-$d$ tensors:} The set of order-$d$ tensors of size $(n_1, \ldots, n_d)$ is denoted $\sR^{n_1\times\dots\times n_d}$. The scalar $T_{p_1\ldots p_d}$ or $[\gT]_{p_1\ldots p_d}$ denotes the $(p_1,\ldots, p_d)$ entry of $\gT\in \sR^{n_1\times\dots\times n_d}$.

\textbf{Rank-$r$ tensors:} A tensor $\gT$ is said to be of rank-one if it can be represented as the outer product of $d$ real-valued vectors $(\vx_1,\ldots, \vx_d)\in \sR^{n_1}\times \cdots \times \sR^{n_d} $. In this case, we write $\gT = \vx_1\otimes \cdots \otimes \vx_d$, where the outer product is defined so that $[\vx_1\otimes \cdots \otimes \vx_d]_{p_1, \ldots, p_d} = \prod_{\ell=1}^d [\vx_\ell]_{p_\ell}$. More generally, a tensor $\gT \in \sR^{n_1\times\dots\times n_d}$ is of rank $r$ if it can be expressed as the sum of $r$ rank-one terms but not less, and is written as $\gT = \sum_{i=1}^r \vx_{i, 1} \otimes \cdots \otimes \vx_{i, d}$, where $\vx_{i,\ell} \in \sR^{n_\ell}$ for all $(i,\ell)\in [r]\times [d]$.

\textbf{Tensor contractions:} For any vectors $\vu_{1}\in \sR^{n_1},\dots, \vu_{d}\in \sR^{n_d}$, contractions of a tensor $\gA$ are denoted by $\gA(\vu_{1},\dots,\vu_d)=\sum A_{p_1\ldots p_d} \prod_{k=1}^d [\vu_{k}]_{p_k} \in \sR$. 
Partial contractions are denoted similarly but with some arguments replaced by a dot. For instance, $\gA(\vu_1, \ldots, \vu_{\ell-1}, \cdot, \vu_{\ell+1}, \dots, \vu_d) \in \sR^{n_\ell}$.

 %\todo{ajouter la contraction sur certains modes seulement, et les dimensions du résultat pour plus de clarté}

\textbf{Tensor norms:} The inner product in $\sR^{n_1\times\dots\times n_d}$ is defined as $\langle \gA, \gB \rangle \equiv \sum_{p_1,\ldots, p_d} A_{p_1\ldots p_d}B_{p_1\ldots p_d}$, the $\ell_2$-norm or Frobenius norm as $\Vert \gA\Vert_F \equiv \sqrt{ \langle \gA, \gA \rangle }$, and the spectral norm as $\Vert \gA\Vert \equiv \supp_{\Vert \vu_\ell \Vert = 1} \vert \gA(\vu_{1},\dots,\vu_d)\vert$.

\textbf{Best rank-one approximation and tensor power iteration:} A best rank-one approximation of $\gT$ corresponds to a rank-one tensor $\lambda \vu_1\otimes \cdots \otimes \vu_d$, where $\lambda>0$ and $\vu_1,\ldots,\vu_d$ are unitary vectors, that minimizes $\Vert \gT - \lambda \vu_1\otimes \cdots \otimes \vu_d\Vert_F^2$.
In particular, due to the Karush-Kuhn-Tucker conditions, the tuple $(\lambda , \vu_1, \ldots, \vu_d)$ satisfies for each $i\in [d]$:
\begin{equation}\label{eq_singular_value_vectors_identities}
    \begin{cases}
    \gT(\vu_1, \ldots, \vu_{\ell-1}, \cdot, \vu_{\ell+1}, \dots, \vu_d) = \lambda \vu_\ell, \\
     \lambda = \gT(\vu_1, \ldots, \vu_d).
    \end{cases}
\end{equation}
These identities generalize the concept of singular value and vectors to tensors \cite{lim2005singular} (it is easy to check that \eqref{eq_singular_value_vectors_identities} yields the classical definition of a matrix singular vector for $d=2$); the scalar $\lambda$ coincides with the spectral norm of $\gT$.

Numerically, a rank-one approximation can be computed approximately via \textit{tensor power iteration} whose $t$-th iteration consists in performing for $\ell = 1,\ldots,d$ the update
\begin{align*}
    \vu_\ell^{(t+1)} = 
     \phi\left(
    \gT(\vu_1^{(t)}, \ldots, \vu_{\ell-1}^{(t)}, \cdot, \vu_{\ell+1}^{(t)}, \ldots, \vu_d^{(t)}) \right),
%     \vu_\ell^{(t+1)} = \frac{\gT(\vu_1^{(t)}, \ldots, \vu_{\ell-1}^{(t)}, \cdot, \vu_{\ell+1}^{(t)}, \ldots, \vu_d^{(t)})}{ \Vert \gT(\vu_1^{(t)}, \ldots, \vu_{\ell-1}^{(t)}, \cdot, \vu_{\ell+1}^{(t)}, \ldots, \vu_d^{(t)})\Vert },
\end{align*}
where $\phi(\vx) = \|\vx\|^{-1} \vx$, 
starting from some initial value \cite{Lathauwer000bestrank1}.
% \todo{H : Si on dit "best" ici, alors on ne peut pas avoir une initialisation arbitraire (best = global max, alors que le problème est non-convexe). Donc soit on supprime "best", soit on dit "a suitable initialization". Perso je préfère la 1ère solution.}

\subsection{\uppercase{Random Matrix Theory Background}}
In this section, we recall some standard tools from random \emph{matrix} theory (RMT) which are at the heart of our main results.
In essence, RMT focuses on describing the distribution of eigenvalues of random matrices.
Specifically, we consider the \textit{resolvent} formalism \cite{hachem2007deterministic} which allows one to characterize the spectral behavior of large symmetric random matrices.
Given a symmetric matrix $\mS\in \sR^{n\times n}$, the resolvent of $\mS$ is defined by $\mR(z) = \left( \mS - z \mI_n \right)^{-1}$ for $z\in \sC\setminus \spec(\mS)$, where $\spec(\mS)$ denotes the set of eigenvalues of $\mS$.

In a celebrated result \cite{Marcenko_Pastur_1967}, Mar\v{c}enko and Pastur showed that under certain technical assumptions on some random matrix $\mS\in \sR^{n\times n}$ with eigenvalues $\lambda_1, \ldots, \lambda_n$, the \textit{empirical spectral measure} of $\mS$, defined as $\hat \mu = \frac1n \sum_{i=1}^n \delta_{\lambda_i}$ with $\delta_x$ the Dirac measure at $x\in \mathbb{R}$, converges in the weak sense \cite{van1996weak} to some deterministic probability measure $\mu$ as $n\to \infty$; RMT aims at describing such $\mu$. To this end, one of the widely considered approaches relies on the \textit{Stieltjes transform} \cite{tao2012topics}. Given a probability measure $\mu$, let $\gS(\mu)$ denote its support. The Stieltjes transform of $\mu$ is defined as $g_\mu(z) = \int \frac{d\mu(\lambda)}{\lambda - z}$ with $z\in \sC\setminus \supp(\mu)$, and the inverse formula allows one to describe the density of $\mu$ (if it exists) as $\mu(dx) = \frac{1}{\pi} \lim_{\varepsilon\to 0} \Im[g_\mu(x+i\varepsilon)]$.

The Stieltjes transform of the empirical spectral measure, denoted by $\hat\mu$, is closely related to the resolvent of $\mS$ through the normalized trace operator. In fact, $g_{\hat\mu}(z) = \frac1n \Tr \mR(z)$ and the \textit{almost sure} convergence of $g_{\hat\mu}(z)$ to some deterministic Stieltjes transform $g(z)$ is equivalent to the weak convergence between the underlying probability measures \cite{tao2012topics}. Our analysis relies on estimating quantities involving $\frac1n \Tr \mR(z)$, making the use of the resolvent approach a natural choice.

\section{\uppercase{Model and Main Results}}
We start by formally describing our considered model. We let $r\geq 1$ and $d\geq 3$, and consider the following rank-$r$ order-$d$ spiked tensor model
\begin{align}\label{main_model}
% \textstyle
\gT_1 = \sum_{i=1}^r \beta_i \, \vx_{i,1} \otimes \cdots \otimes \vx_{i,d} + \frac{1}{\sqrt N} \mathcal \gW,
\end{align}
where $W_{p_1\ldots p_d} \overset{\text{i.i.d.}}{\sim} \gN(0, 1)$, $\vx_{i, \ell} \in\sR^{n_\ell}$ are \emph{correlated} unit-norm vectors {(i.e., $\langle \vx_{i, \ell}, \vx_{j, \ell} \rangle \neq 0$ for $i \neq j$)}, $N = \sum_{\ell=1}^d n_\ell$ and $\beta_i \geq 0$.
\paragraph*{Tensor deflation:} For tensors, Hotelling deflation is implemented by computing $\gT_{2},\gT_{3},\dots$ sequentially through 
%\todo{Un peu de redondance avec l'Intro ici, on peut faire plus court}
\begin{align}
\gT_{i+1} = \gT_i - \hat\lambda_i \hat\vu_{i, 1} \otimes \cdots \otimes \hat\vu_{i, d} \quad \mathrm{for} \quad i\in [r],
\end{align}
where $\hat\lambda_i \, \hat\vu_{i, 1} \otimes \cdots \otimes \hat\vu_{i, d}$ is a 
(in practice, local) minimizer of $\big\Vert \gT_i -  \lambda_i \vu_{i, 1} \otimes \cdots \otimes \vu_{i, d} \big\Vert_F^2$ under the constraints $\lambda_i > 0$ and $\|\vu_{i, \ell}\| = 1$ for all $(i,\ell) \in [r]\times[d]$, thus satisfying
% critical point of {\blue} $\big\Vert \gT_i -  \lambda_i \vu_{i, 1} \otimes \cdots \otimes \vu_{i, d} \big\Vert_F^2$ which corresponds to the best rank-one approximation of $\gT_i$ \cite{lim2005singular}.
%Equivalently, the MLE can be formulated through a variational approach \cite{lim2005singular} as $\max_{\Vert \vu_{i,j} \Vert = 1 } \left\vert \left\langle \gT_i, \vu_{i, 1} \otimes \cdots \otimes \vu_{i, d} \right\rangle \right\vert$.
% The critical points satisfy the Karush-Kuhn-Tucker conditions derived from the Lagrangian of the latter objective, i.e.
\begin{align}\label{eq_kkt}
\gT_i (\hat\vu_{i, 1}, \ldots, \hat\vu_{i, \ell-1}, \cdot, \hat\vu_{i, \ell+1} ,\ldots, \hat\vu_{i, d} )= \hat\lambda_i \hat\vu_{i, \ell}
\end{align}
for all $\ell\in[d]$.
%In practice, we use tensor power iteration in order to compute numerical solutions.
In the following, we aim at computing the limits of $\hat\lambda_i$ and $\langle\vx_{i,\ell}, \hat\vu_{j,\ell}\rangle $ for $i,j\in[r]$ and $\ell\in[d]$ when the dimensions $n_i$ grow large.

\subsection{\uppercase{Sketch of the analytical approach}}
We follow the approach developed in \cite{seddik2021random}, whereby each tensor $\gT_i $ is associated to a structured random matrix given by $\mM = \mPhi_d(\gT_i, \hat\vu_{i,1}, \ldots, \hat\vu_{i,d}) \in \sR^{N\times N}$, where 
\begin{align}\label{eq_mapping_Phi}
    \begin{split}
            \mPhi_d: \begin{pmatrix}
                \tX \\
                \va_1 \\
                \vdots \\
                \va_d
            \end{pmatrix} & \longmapsto \begin{bmatrix}
\vzero_{n_1\times n_1} & \tX^{12} &  \cdots & \tX^{1d}\\
(\tX^{12})^\top & \vzero_{n_2\times n_2}  & \cdots & \tX^{2d}\\
\vdots & \vdots  & \ddots & \vdots \\
(\tX^{1d})^\top & (\tX^{2d})^\top  & \cdots & \vzero_{n_d\times n_d}
\end{bmatrix}
    \end{split}
\end{align}
% in \eqref{eq_mapping_Phi} \vpageref{eq_mapping_Phi}
and each block $\tX^{\ell m}$ is an $n_\ell\times n_m$ matrix resulting from the contraction of $\tX$ on $d-2$ vectors, namely 
\[\tX^{\ell m} \equiv \tX( \va_1, \ldots, \va_{\ell-1}, \cdot, \va_{\ell+1}, \ldots, \va_{m-1}, \cdot, \va_{m+1}, \ldots, \va_d ).\]
We further denote $\mR(z) = ( \mM - z \mI_N )^{-1}$ the resolvent of $\mM$.
%We show that the spectral characteristics of $\mM$ are tightly coupled to the deflation problem at each iteration.
Then, the characterization of the limits of $\hat\lambda_i$ and the alignments $\langle\vx_{i,\ell}, \hat\vu_{j,\ell}\rangle $ when $n_\ell\to \infty$ requires the computation of the Stieltjes transform of the limiting spectral measure of $\mPhi_d(\gT_i, \hat\vu_{i,1}, \ldots, \hat\vu_{i,d})$. Hence, we need the following definition and technical assumptions.
%\todo{H : je dirais plutôt "requires" que "boils down", non ? Il faut la TS, mais d'autres calculs sont nécessaires aussi ?}

\begin{assumption}\label{assumptions} Let us assume that $n_\ell\to \infty$ for all $\ell\in[d]$, that the limits $c_\ell = \lim \frac{n_\ell}{\sum_{m=1}^d n_m}$ exist, and that $r=O(1)$. For notational convenience, in the sequel the notation $\lim Q$ stands for the limit of the quantity $Q$ in this regime. We further assume that the correlation between the $\vx_{i,\ell}$ is controlled by $\lim \vert\langle \vx_{i,\ell}, \vx_{j,\ell} \rangle\vert = \alpha_{ij,\ell} \in [0, 1]$, and that there exists a sequence of critical points such that $\hat\lambda_i\as \lambda_i$, $ \vert\langle \vx_{i,\ell}, \hat\vu_{j,\ell} \rangle\vert\as \rho_{ij,\ell}$ and $\vert\langle \hat\vu_{i,\ell}, \hat\vu_{j,\ell} \rangle\vert\as\eta_{ij,\ell}$ such that $\lambda_i \notin \gS(\mu)$  and $\rho_{ij,\ell} > 0$.
\end{assumption}

\begin{definition}\label{def_measure} Let $\mu$ be the probability measure with Stieltjes transform $g(z) = \sum_{\ell=1}^d g_\ell(z)$ verifying $\Im[g(z)] > 0$ for $\Im[z] > 0$, where $g_\ell(z)$ satisfies $g_\ell^2(z) - (g(z) + z) g_\ell(z) - c_\ell = 0$, for $z\notin \gS(\mu)$ and $\gS(\mu)$ stands for the support of $\mu$.
\end{definition}

The following result characterizes the limiting spectral measure of $\mPhi_d(\gT_i, \hat\vu_{i,1}, \ldots, \hat\vu_{i,d})$:
\begin{theorem}\label{theorem_spectrum} Under Assumption \ref{assumptions}, the empirical spectral measure of $\mPhi_d(\gT_i, \hat\vu_{i,1}, \ldots, \hat\vu_{i,d})$ converges to the deterministic measure $\mu$ defined in Definition \ref{def_measure}. \newpage\noindent %newpage because EDAS complains about the margin
In particular, $\frac1N \Tr \mR(z) \asto g(z)$.
\end{theorem}
%\todo{H : on a déjà utilisé la notation $\asto$ avant. À définir plus haut ?}

%\begin{proof}[Sketch of the proof]
The proof follows similar arguments as in \cite{seddik2021random} and requires some additional arguments for controlling the statistical dependencies between the $\hat\vu_{i,\ell}$'s and the noise $\gW$.
In fact, the result is identical to that of \cite{seddik2021random} because $r=O(1)$, since the contribution of the rank-one terms vanish asymptotically.
% \end{proof}
% \todo{H : Ce n'est même pas un sketch, à mon avis. On peut juste faire le commentaire sans dire que c'est un sketch ? Je pense que c'est intéressant ici surtout de souligner les éventuelles différences avec le résultat de \cite{seddik2021random} (non pas seulement pour la preuve mais aussi pour l'énoncé). En effet, dans la suite on fait appel au résultat de \cite{seddik2021random} pour donner une expression pour la TS de $\mu$. Cela laisse penser que les résultats sont identiques.}
As shown in \cite{seddik2021random}, in the case $c_\ell=\frac1d$ for all $\ell\in [d]$, the measure $\mu$ describes a semi-circle or Wigner-type law of compact support $\gS(\mu) = [-2\sqrt{ \frac{d-1}{d} }, 2\sqrt{ \frac{d-1}{d} }]$, the Stieltjes transform of which writes explicitly as
\begin{align}\label{Stieltjes_transform_g}
g(z) = \frac{-zd + d \sqrt{ z^2 - \frac{ 4(d-1) }{d} } }{ 2 (d-1)}, \quad z\notin \gS(\mu).
\end{align}

\paragraph*{Limiting spectral norms and alignments:} We introduce the quantities $f(z) = z + g(z)$ and $h_\ell(z) = - \frac{c_\ell}{g_\ell(z) }$ that shall be used subsequently. The main result brought by this paper describes the asymptotic singular values and alignments obtained after each tensor deflation step as stated by the following theorem.
\begin{theorem}\label{main_theorem}
Assume that Assumption \ref{assumptions} holds. Then, $\lambda_i$, $\rho_{ij,\ell}$ and $\eta_{ij,\ell}$ satisfy the following system of equations
\begin{small}
\begin{align*}
\begin{dcases}
f( \lambda_i) + \sum_{k=1}^{i-1}  \lambda_k \prod_{\ell=1}^d \eta_{k i,\ell} - \sum_{k=1}^r \beta_k \prod_{\ell=1}^d \rho_{ki,\ell} = 0 \text{\ , $1\leq i \leq r$}\\
h_\ell( \lambda_i ) \rho_{ji,\ell} + \sum_{k=1}^{i-1} \lambda_k \rho_{jk,\ell} \prod_{m\neq\ell}^d \eta_{ki,m} - \sum_{k=1}^r \beta_k \alpha_{kj,\ell} \prod_{m\neq\ell}^d \rho_{ki,m} = 0\\ \text{ $1\leq \ell\leq d, 1\leq i,j \leq r$}\\
h_\ell( \lambda_i ) \eta_{ji,\ell} + g_\ell(\lambda_j) \prod_{m\neq \ell}^d \eta_{ji,m} + \sum_{k=1}^{i-1} \lambda_k \eta_{kj,\ell} \prod_{m\neq \ell}^d \eta_{ki,m}  +\ldots \\ - \sum_{k=1}^r \beta_k \rho_{kj,\ell} \prod_{m\neq \ell}^d \rho_{ki,m} = 0 \text{,\ $1\leq \ell\leq d, 1\leq i<j \leq r$}\\
\end{dcases}
\end{align*}
\end{small}
\end{theorem}
\begin{proof}[Sketch of the proof] We use similar arguments as in \cite{seddik2021random}:
% $\Var[\hat\lambda_i] = O(n^{-1})$ and use a concentration argument to show that $\hat\lambda_i$ concentrates around its expectation. Similarly, the same property holds for the alignments.
% \todo{H : Pourquoi a-t-on besoin de montrer la concentration ? Assumption 3.1 l'impose déjà avec la convergence p.s.}
we evaluate the expectation of the scalar product between \eqref{eq_kkt} and $\vx_{i,\ell}$ or $\hat\vu_{j,\ell}$ using Stein's Lemma.\footnote{$\mathbb{E}[Wf(W)] = \mathbb{E}[f'(W)]$ for $W\sim\mathcal{N}(0, 1)$ and a continuously differentiable $f$ having at most polynomial growth.}
Along the proof, the variances of $\hat\lambda_i$ and of the alignments $\rho_{ij,\ell}$, $\eta_{ij,\ell}$ are shown to decay as $O(N^{-1})$.
\end{proof}

\begin{figure*}[t!]
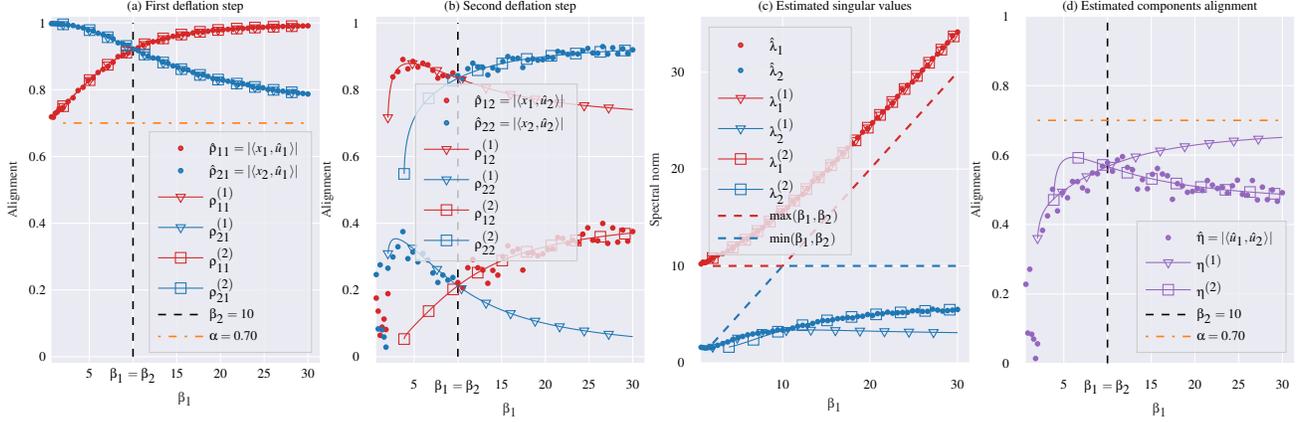

\tiny
\include{figures/deflation_fig.tex}
\vspace{-0.5cm}
%\begin{tikzpicture}
%\node[anchor=south west] (img) {\includegraphics[width=\textwidth]{figures/deflation_fig.pdf}};
%\begin{scope}[x=(img.south east),y=(img.north west)]
%        \node[draw,minimum height=0.4cm,minimum width=0.35cm] (B1) at (0.845,0.60) {};
%        \node[draw,minimum height=0.4cm,minimum width=0.35cm] (B2) at (0.600,0.28) {};
%        \node[draw,minimum height=0.4cm,minimum width=0.35cm] (B3) at (0.355,0.18) {};
%\end{scope}
%\node (img1)  at (16.85,1.35) {\includegraphics[width=0.05\linewidth]{figures/zoom_fig3d.png}};
%\draw (img1.south west) rectangle (img1.north east);
%\draw (B1) -- (img1);
%\node (img2)  at (12.85,1.32) {\includegraphics[width=0.03\linewidth]{figures/zoom_fig3c.png}};
%\draw (img2.south west) rectangle (img2.north east);
%\draw (B2) -- (img2);
%\node (img3)  at (8.65,1.34) {\includegraphics[width=0.025\linewidth]{figures/zoom_fig3b.png}};
%\draw (img3.south west) rectangle (img3.north east);
%\draw (B3) -- (img3);
%\end{tikzpicture}
%\includegraphics[width=\textwidth]{figures/deflation_fig.pdf}
\caption{Illustration of $\langle \vx_i, \hat\vu_j \rangle$ for $(i,j)\in\{1, 2 \}$, $\langle \hat\vu_1, \hat\vu_2 \rangle$, $\hat\lambda_1$ and $\hat\lambda_2$ vs. their limits for $\beta_2=10$ and $\alpha = \langle \vx_1, \vx_2 \rangle = 0.7$ of the two spikes tensor model in \eqref{eq_two_spikes_model}. (a) shows the alignments between the signal components and the first singular vectors corresponding to the best rank-one approximation of $\gT_1$. (b) shows the alignments with the second singular vectors computed after deflation. (c) depicts the singular values. (d) shows the alignments between the singular vectors computed at each step of the deflation procedure. Simulations were performed on a tensor of dimensions $(50, 50, 50)$.}
\label{figure_alignments}
\end{figure*}

\subsection{\uppercase{Particular case of a rank-2, order-3 tensor}}
For the sake of clarity, let us consider the example of a rank-$2$ order-$3$ spiked tensor with $n_1=n_2=n_3$, for which we can get more insightful analytical results. The model in \eqref{main_model} becomes
\begin{align}\label{eq_two_spikes_model}
\gT_1 = \sum_{i=1}^2 \beta_i \, \vx_{i, 1} \otimes \vx_{i, 2} \otimes \vx_{i, 3} + \frac{1}{\sqrt{N}}\gW.
\end{align}
Furthermore, we assume that for all $i\neq j$ and each $\ell\in[3]$, $\lim\vert \langle \vx_{i,\ell}, \vx_{j,\ell} \rangle \vert = \alpha\in [0, 1]$. In this case, since all the dimensions $n_i$ are equal, the limits of $\vert \langle \vx_{i,\ell}, \hat\vu_{j,\ell} \rangle \vert$ and of $\vert \langle \hat\vu_{1,\ell}, \hat\vu_{2,\ell} \rangle \vert$ are both independent of $\ell$ by symmetry. Therefore, we drop their dependence on $\ell$ in our notations. The system of equations in Theorem \ref{main_theorem} reduces to seven equations, as detailed in the following corollary:
\begin{corollary}\label{corollary} Denote $\rho_{ij} = \lim \vert \langle \vx_{i,\ell}, \hat\vu_{j,\ell} \rangle \vert$ for $i,j\in [2]$ and $\eta = \lim \vert \langle \hat\vu_{1,\ell}, \hat\vu_{2,\ell} \rangle \vert$ and suppose that Assumption \ref{assumptions} holds, then $ \lambda_i$, $\rho_{ij}$ and $\eta$ satisfy $\psi(\vlambda,\vbeta,\vrho)=\boldsymbol{0}$ with $\vlambda=(\lambda_1,\lambda_2,\eta)$, $\vrho=(\rho_{11},\rho_{12},\rho_{21},\rho_{22})$, $\vbeta=(\beta_1,\beta_2,\alpha)$ and
\begin{small}
\begin{equation}\label{eq_system_two_spikes}
\psi(\vlambda,\vbeta,\vrho) =
\left(\begin{array}{c}
f(\lambda_1) - \beta_1 \rho_{11}^3 - \beta_2 \rho_{21}^3 \\
h( \lambda_1) \rho_{11} - \beta_1 \rho_{11}^2 - \beta_2 \alpha \rho_{21}^2\\
h( \lambda_1) \rho_{21} - \beta_1 \alpha \rho_{11}^2 - \beta_2 \rho_{21}^2\\
f( \lambda_2 ) +  \lambda_1 \eta^3 - \beta_1 \rho_{12}^3 - \beta_2 \rho_{22}^3\\
h( \lambda_2 ) \rho_{12} +  \lambda_1 \rho_{11} \eta^2 - \beta_1 \rho_{12}^2 - \beta_2 \alpha \rho_{22}^2 \\
h( \lambda_2 ) \rho_{22} +  \lambda_1 \rho_{21} \eta^2 - \beta_1 \alpha \rho_{12}^2 - \beta_2 \rho_{22}^2 \\
h( \lambda_2) \eta + q( \lambda_1 ) \eta^2 - \beta_1 \rho_{11} \rho_{12}^2 - \beta_2 \rho_{21} \rho_{22}^2
\end{array}\right)
\end{equation}
\end{small}
\!where $h(z) = \frac{-1}{g(z)}$ and $q(z) = z + \frac{g(z)}{3}$ with $g(z)$ given by \eqref{Stieltjes_transform_g} for $d=3$ and we recall that $f(z) = z + g(z)$.
\end{corollary}

%\begin{figure}[b!]
%\includegraphics[width=.49\textwidth]{figures/phase_1.pdf}
%\includegraphics[width=.49\textwidth]{figures/phase_2.pdf}
%\caption{Phase diagram of the two spikes model in \eqref{eq_two_spikes_model}. First row corresponds to $\beta_2 = 1$ and second row for $\beta_2=2$. The first column depicts $\lambda_1$ varying $\beta_1$ and $\alpha$, while the remaining columns depict the asymptotic alignments between the $\vx_i$'s and $\hat\vu_1$. The figures were obtained by solving the first three equations in \eqref{eq_system_two_spikes}.}
%\label{fig_phase}
%\end{figure}

\begin{comment}
\begin{figure}[b!]
\includegraphics[width=.49\textwidth]{figures/estimation_beta_n30.pdf}
\includegraphics[width=.49\textwidth]{figures/estimation_beta_n100.pdf}
\caption{Estimation of the $\beta_i$'s by solving the system $\psi(\hat\vlambda,\cdot,\cdot)=\boldsymbol{0}$ given $\hat\lambda_1$, $\hat\lambda_2$ and $\hat \eta = \langle \hat\vu_1, \hat\vu_2 \rangle$ estimated from a two steps deflation on a tensor distributed as in \eqref{eq_two_spikes_model}. First row corresponds to $n_i=30$ and second row to $n_i=100$.}
\label{fig_estimation_beta}
\end{figure}
\end{comment}

Fixing $\vbeta=(\beta_1,\beta_2,\alpha)$, one can solve $\psi(\vlambda,\vbeta,\vrho)=\boldsymbol{0}$  in $(\vlambda,\vrho)$  using a Newton-Raphson method with $100$ random initializations. We keep the solutions ensuring that $0\leq \eta,\rho_{ij} \leq 1$ and $\lambda_1, \lambda_2> 2\sqrt{\frac{2}{3}}$.
Note that the equation $\psi(\cdot,\vbeta,\cdot)=\boldsymbol{0}$ can have several distinct solutions. These different solutions correspond to different sequences of critical points. 

%Note that the first three equations in \eqref{eq_system_two_spikes} only involve $ \lambda_1$, $\rho_{11}$ and $\rho_{21}$ and are decoupled from the last four equations. Therefore, solving them allows to obtain the \textit{phase diagram} related to the dominant singular mode $(\hat\lambda_1,\hat\vu_1)$, depicted in Figure~\ref{fig_phase}. 
%We observed however in practice that all solutions have the same values for $(\lambda_2,\rho_{11},\rho_{12})$.
%Therefore, we illustrate their \textit{phase diagrams} in Figure~\ref{fig_phase}. 
%It suggests that when $\beta_2$ is not large enough (e.g. $\beta_2=1$, top row), there exists a region (varying $\beta_1$ and $\alpha$) where it is information-theoretically impossible to detect a signal, while outside this region estimation may be possible with the MLE (in that case, the estimated $\hat\vu_1$ is shown to be correlated with both $\vx_1$ and $\vx_2$).
%Moreover, in this region, the spectral norm obtained at the first step of the deflation method is bounded by $2\sqrt{ \frac{2}{3} }$ which provides a criterion for signal detection, i.e. a stopping criterion of the deflation method.
%For $\beta_2$ sufficiently large (e.g. $\beta_2=2$, the bottom row of Figure~\ref{fig_phase}), signal detection is always possible and the singular vector $\hat\vu_1$ presents a higher alignment with the signal components having the highest SNR $\beta_i$ (see the two columns on the right).

Indeed, we illustrate in Figure~\ref{figure_alignments} the matching between $\hat\vlambda=(\hat\lambda_1, \hat\lambda_2, \hat\eta )$, $\hat\vrho=(\hat\rho_{11},\hat\rho_{12},\hat\rho_{21},\hat\rho_{22})$ (where the rank-one approximations are performed using tensor power iteration initialized by tensor SVD \cite{auddy2022estimating}) and their asymptotic limits $\vlambda,\vrho$.
We observe that, when $\alpha=0.7$, the equation $\psi(\cdot,\vbeta,\cdot)=\boldsymbol{0}$ has two solutions (we note $(\vrho^{(1)}, \vlambda^{(1)})$ the solution with the highest $\rho_{12}$ and $(\vrho^{(2)}, \vlambda^{(2)})$ the other solution) and
our initialization of the tensor power method seems to favor the global minimizer. Moreover, note that for $\beta_1$ large enough (e.g. $\beta_1>2 \sqrt{2/3}$, as per Figure~\ref{figure_alignments}), signal detection is always possible and the singular vector $\hat\vu_1$ presents a higher alignment with the signal components having the highest $\beta_i$. However, when $\beta_1$ is not large enough, there exists a region (see in particular Figure~\ref{figure_alignments}(b)) where it is information-theoretically impossible to detect a signal, while outside this region estimation becomes possible with the MLE. We refer the reader to \cite{goulart2021random, seddik2021random} for detailed discussions about the detectability threshold in the rank-one case.

\section{Conclusion and discussion}
We have provided an analysis of a tensor deflation method in the high-dimensional regime and assumed a low-rank spiked tensor model with correlated signal components. Our analysis allows for a precise description of the asymptotic behavior of such models. Moreover, our results can be exploited to design an estimator of the model parameters by solving the system $\psi(\hat\vlambda,\cdot,\cdot)=\boldsymbol{0}$, therefore, possibly resulting in improved deflation.
%{\red and provides a consistent estimation of its parameters as shown in the last part of the paper.} 
This paves a new way for the analysis of more sophisticated tensor decomposition methods and the understanding of more general tensor models through random tensor theory.

%Our current results still have some limitations. In particular, we consider Assumption \ref{assumptions} which is standard in the analysis of large spiked random tensors \cite{goulart2021random} and is not specific to the considered low-rank spiked tensor model, but the almost sure convergences stated therein need to be proved through concentration results. 
Our current results still have some limitations. In particular, we consider Assumption \ref{assumptions} which is standard in the analysis of large spiked random tensors \cite{goulart2021random} but the almost sure convergences stated therein need to be proved through concentration results. 
%{\red Moreover, we highlight that our current findings do not contain proof of consistency for SNR estimation introduced in the previous section, but such consistency is supported by empirical evidence as we showed in Figure \ref{fig_estimation_beta}. Indeed, we believe that consistency can be ensured with additional assumptions on the function $\psi$.} 
%Also, the existence and uniqueness of the solutions of the involved systems of equations are not discussed in our present paper and are therefore left for future investigation.
Also, the existence and uniqueness of the solutions of the involved systems of equations are not discussed in this paper and are left for future investigation.

\bibliographystyle{IEEEbib}
{\small
\bibliography{biblio}}

\end{document}